\documentclass[11pt]{article}
\usepackage{graphicx,psfrag,epsf}
\usepackage{enumerate}
\usepackage{natbib}
\usepackage{url} 
\usepackage{amssymb, amsmath, amsthm,graphicx,bbm,float}
\usepackage[linesnumbered,ruled]{algorithm2e}
\usepackage{algorithmic}
\usepackage{color}
\usepackage[toc,page]{appendix}

\DeclareMathOperator*{\argmin}{argmin}
\newtheorem{theorem}{Theorem}[section]

\newtheorem{proposition}[theorem]{Proposition}
\usepackage{framed, color,xcolor}
\definecolor{darkred}{rgb}{0.8, 0, 0}
\definecolor{darkblue}{rgb}{0, 0, 205}

\newcommand{\blind}{0}

\begin{document}

\def\spacingset#1{\renewcommand{\baselinestretch}%
{#1}\small\normalsize} \spacingset{1}

\if0\blind
{
  \title{\bf KTBoost: Combined Kernel and Tree Boosting}
  
  \author{Fabio Sigrist\thanks{Email: fabio.sigrist@hslu.ch. Address: Lucerne University of Applied Sciences and Arts, Suurstoffi 1, 6343 Rotkreuz, Switzerland.}}
  \maketitle
}
\fi

\bigskip

\spacingset{1} 

\bibliographystyle{abbrvnat}

\begin{abstract}
	We introduce a novel boosting algorithm called `KTBoost' which combines \textbf{k}ernel boosting and \textbf{t}ree boosting. In each boosting iteration, the algorithm adds either a regression tree or reproducing kernel Hilbert space (RKHS) regression function to the ensemble of base learners. Intuitively, the idea is that discontinuous trees and continuous RKHS regression functions complement each other, and that this combination allows for better learning of functions that have parts with varying degrees of regularity such as discontinuities and smooth parts. We empirically show that KTBoost significantly outperforms both tree and kernel boosting in terms of predictive accuracy in a comparison on a wide array of data sets.
	
\end{abstract}

\section{Introduction}\label{intro}
Boosting algorithms \citep{freund1996experiments,friedman2000additive,mason2000boosting,friedman2001greedy,buhlmann2007boosting} enjoy large popularity in both applied data science and machine learning research, among other things, due to their high predictive accuracy observed on a wide range of data sets \citep{chen2016xgboost}. Boosting additively combines base learners by sequentially minimizing a risk functional. Despite the fact that there is almost no restriction on the type of base learners in the seminal papers of \citet{freund1996experiments} and \citet{freund1997decision}, very little research has been done on combining different types of base learners. To the best of our knowledge, except for one reference \citep{hothorn2010model}, existing boosting algorithms use only one type of functions as base learners. To date, regression trees are the most common choice of base learners, and a lot of effort has been made in recent years to develop tree-based boosting methods that scale to large data \citep{chen2016xgboost,ke2017lightgbm,ponomareva2017tf,CatBoost2017}. 

In this article, we relax the assumption of using only one type of base learners by combining regression trees \citep{breiman1984classification} and reproducing kernel Hilbert space (RKHS) regression functions \citep{scholkopf2001learning,berlinet2011reproducing} as base learners. In short, RKHS regression is a form of non-parametric regression which shows state-of-the-art predictive accuracy for many data sets as it can, for instance, achieve near-optimal test errors \citep{belkin18a,belkin18b}, and kernel classifiers parallel the behaviors of deep networks as noted in \citet{zhang2016understanding}. As there is now growing evidence that base learners do not necessarily need to have low complexity \cite[e.g.][]{wyner2017explaining}, continuous, or smooth, RKHS functions have thus the potential to complement discontinuous trees as base learners.

\subsection{Summary of results}
We introduce a novel boosting algorithm denoted by `KTBoost' which combines \textbf{k}ernel and \textbf{t}ree boosting. In each boosting iteration, the KTBoost algorithm adds either a regression tree or a penalized RKHS regression function, also known as kernel ridge regression \citep{Murphy2012}, to the ensemble. This is done by first learning both a tree and an RKHS function using one step of functional Newton's method or functional gradient descent, and then selecting the base learner whose addition to the ensemble results in the lowest empirical risk. The KTBoost algorithm thus chooses in each iteration a base learner from two fundamentally different function classes. Functions in an RKHS are continuous and, depending on the kernel function, they also have higher regularity. Trees, on the other hand, are discontinuous functions. 

\begin{figure}[ht!]
	\centering
	\includegraphics[width=0.7\columnwidth]{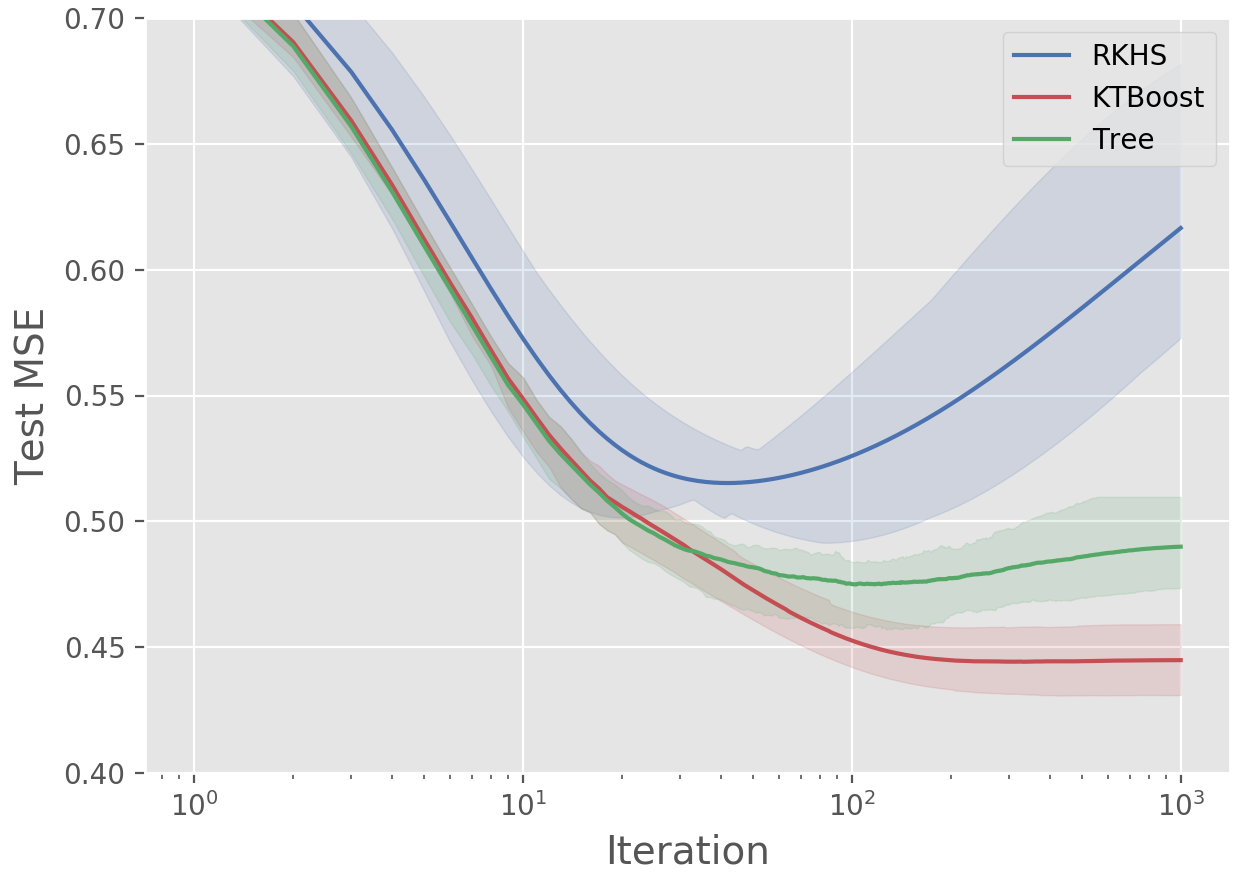}
	\caption{Test mean square error (MSE) versus the number of boosting iteration for KTBoost in comparison with tree and kernel boosting for one data set (wine).}
	\label{traceplot}
\end{figure}

Intuitively, the idea is that the different types of base learners complement each other, and that this combination allows for better learning of functions that exhibit parts with varying degrees of regularity. We demonstrate this effect in a simulation study in Section \ref{simstudysec}. To briefly illustrate that the combination of trees and RKHS functions as base learners can achieve higher predictive accuracy, we report in Figure \ref{traceplot} test mean square errors (MSEs) versus the number of boosting iterations for one data set (wine). The solid lines show average test MSEs over ten random splits into training, validation, and test data sets versus the number of boosting iterations. The confidence bands are obtained after point-wise excluding the largest and smallest MSEs. Tuning parameters of all methods are chosen on the validation data sets. See Section \ref{exp} for more details on the data set and the choice of tuning parameters.\footnote{For better comparison, the shrinkage parameter $\nu$, see Equation \eqref{shrinkage}, is set to a fix value ($\nu=0.1$) in this example. In the experiments in Section \ref{realworlddata}, the shrinkage parameter is also chosen using cross-validation.} The figure illustrates how the combination of tree and kernel boosting (KTBoost) results in a lower test MSE compared to both tree and kernel boosting. In our extensive experiments in Section \ref{realworlddata}, we show on a large collection of data sets that the combination of trees and RKHS functions leads to a lower generalization error compared to both only tree and only kernel boosting. Our approach is implemented in the Python package \texttt{KTBoost} which is openly available on the Python Package Index (PyPI) repository.\footnote{See https://github.com/fabsig/KTBoost for more information.}

\subsection{Related work}
Combining predictions from several models has been successfully applied in many areas of machine learning such as diversity inducing methods \citep[e.g.][]{mendes2012ensemble} or multi-view learning; see e.g. \citet{peng2018multiview} for a recent example of a boosting application. However, the way boosting combines base learners is different from traditional ensembles consisting of several models trained on potentially different data sets since, for instance, boosting reduces both variance and bias. Very little research has been done on combining different types of base learners in a boosting framework, and, to the best of our knowledge, there is no study which investigates the effect on the predictive accuracy when boosting different types of base learners.


The \texttt{mboost} R package of \citet{hothorn2010model} allows for combining different base learners which include linear functions, one- and two-dimensional smoothing splines, spatial terms, regression trees, as well as user-defined ones. This approach is different from ours since \texttt{mboost} uses a component-wise approach where every base learner typically depends on only a few features, and in each boosting update, the term which minimizes a least squares approximation to the negative gradient of the empirical risk is added to the ensemble. In contrast, in our approach, the tree and the kernel machine depend on all features by default, base learners are learned using Newton's method or gradient descent, and we select the base learner whose addition to the ensemble directly results in the lowest empirical risk. 

The idea that machine learning methods should be able learn both smooth as well as non-smooth functions has recently received attention also in other areas of machine learning. For instance, \citet{imaizumi2019deep} and \citet{hayakawa2020minimax} argue that one of the reasons for the superior predictive accuracy of deep neural networks, over e.g. kernel methods, is their ability to also learn non-smooth functions.


\section{Preliminaries}
\subsection{Boosting}
There exist population as well as sample versions of boosting algorithms. For the sake of brevity, we only consider the latter here. Assume that we have data $\{(x_i,y_i)\in \mathbb{R}^p\times\mathbb{R}, i=1,\dots,n\}$ from a probability distribution $P_{X,Y}$. The goal of boosting is to find a function $F:\mathbb{R}^p\rightarrow \mathbb{R}$ for predicting $y$ given $x$, where $F$ is in a function space $\Omega_{\mathcal{S}}$ with inner product $\langle\cdot,\cdot\rangle$ given by $\langle F,F\rangle=E_X\left(F(X)^2\right),$ and the expectation is with respect to the marginal distribution $P_X$ of $P_{X,Y}$. Note that $y$ can be categorical, discrete, continuous, or of mixed type depending on whether the conditional distribution $P_{Y|X}$ is absolutely continuous with respect to the Lebesgue, a counting measure, or a mixture of the two; see, e.g., \citet{sigrist2017grabit} for an example of the latter. Depending on the data and the goal of the application, the function can also be multivariate. For the sake of notational simplicity, we assume in the following that $F$ is univariate. The extension to the multivariate case $F=(F^k),k=1,\dots,d$, is straightforward; see, e.g., \citet{sigrist2018gradient}.


The goal of boosting is to find a minimizer $F^*(\cdot)$ of the empirical risk functional $R(F)$:
\begin{equation}\label{emprisk}
F^*(\cdot)=\argmin_{F(\cdot)\in \Omega_{\mathcal{S}}}R(F)=\argmin_{F(\cdot)\in \Omega_{\mathcal{S}}}\sum_{i=1}^n L(y_i,F(x_i)),
\end{equation}
where $L(Y,F)$ is an appropriately chosen loss function such as the squared error for regression or the logistic regression loss for binary classification, and $\Omega_{\mathcal{S}}=span(\mathcal{S})$ is the span of a set of base learners $\mathcal{S}=\{f_j:\mathbb{R}^p\rightarrow \mathbb{R}\}$. Boosting finds $F^*(\cdot)$ in a sequential way by iteratively adding an update $f_m$ to the current estimate $F_{m-1}$:
\begin{equation}\label{boostupdate}
F_m(x)= F_{m-1}(x)+ f_m(x),~~f_m\in \mathcal{S}, ~~m=1,\dots,M,
\end{equation}
such that the empirical risk is minimized
\begin{equation}\label{minupdt}
f_m=\argmin_{f\in \mathcal{S}}R\left(F_{m-1}+  f\right).
\end{equation}
Since this usually cannot be found explicitly, one uses an approximate minimizer. Depending on whether gradient or Newton boosting is used, the update $f_m$ is either obtained as the least squares approximation to the negative functional gradient or by applying one step of functional Newton's method which corresponds to minimizing a second order Taylor expansion of the risk functional; see Section \ref{ktboost} or \citet{sigrist2018gradient} for more information. For increased predictive accuracy \citep{friedman2001greedy}, an additional shrinkage parameter $\nu>0$ is usually added to the update equation:
\begin{equation}\label{shrinkage}
F_m(x)= F_{m-1}(x)+ \nu f_m(x).
\end{equation}


\subsection{Reproducing kernel Hilbert space regression}\label{krr}
Assume that $K:\mathbb{R}^d\times \mathbb{R}^d\rightarrow \mathbb{R}$ is a positive definite kernel function. Then there exists a reproducing kernel Hilbert space (RKHS) $\mathcal{H}$ with an inner product $\langle\cdot,\cdot\rangle$ such that (i) the function $K(\cdot,x)$ belongs to $\mathcal{H}$ for all $x\in \mathbb{R}^d$ and (ii) $f(x)=\langle f,K(\cdot,x)\rangle$ for all $f\in \mathcal{H}$. Suppose we are interested in finding the minimizer
\begin{equation}\label{minKern}
\argmin_{f\in \mathcal{H}}\sum_{i=1}^n(y_i-f(x_i))^2+\lambda \| f\|^2_{\mathcal{H}},
\end{equation}
where $\lambda\geq0$ is a regularization parameter. The representer theorem \citep{scholkopf2001generalized} then states that there is a unique minimizer of the form $$f(\cdot)=\sum_{j=1}^n\alpha_jK(x_j,\cdot)$$ and \eqref{minKern} can be written as $$\argmin_{\alpha\in \mathbb{R}^n}\|y-K\alpha\|^2+\lambda\alpha^TK\alpha,$$
where $y=(y_1,\dots,y_n)^T$, $K\in\mathbb{R}^{n\times n}$ with $K_{ij}=K(x_i,x_j)$, and $\alpha=(\alpha_1,\dots,\alpha_n)^T$. Taking derivatives and equaling them to zero, we find the explicit solution as
$$\alpha=(K+\lambda I_n)^{-1}y,$$ where $I_n$ denotes the $n$-dimensional identity matrix. 


There is a close connection between Gaussian process regression and kernel regression. The solution to \eqref{minKern} is the posterior mean conditional on the data of a zero-mean Gaussian process with covariance function $K$. Further, since $$f(x)=k(x)^T(K+\lambda I_n)^{-1}y,$$ where 
\begin{equation}\label{kxeq}
k(x)=(K(x_1,x),\dots,K(x_n,x))^T,
\end{equation}
kernel regression can also be interpreted as a two-layer neural network.


\subsection{Regression trees}\label{trees}
We denote by $\mathcal{T}$ the space which consists of regression trees \citep{breiman1984classification}. Following the notation used in \citet{chen2016xgboost}, a regression tree is given by
\[f^T(x)=w_{s(x)},\] where $s:\mathbb{R}^p\rightarrow \{1,\dots,J\}$, $w\in \mathbb{R}^J$, and $J\in \mathbb{N}$ denotes the number of terminal nodes of the tree $f^T(x)$. $s$ determines the structure of the tree, i.e., the partition of the space, and $w$ denotes the leaf values. As in \citet{breiman1984classification}, we assume that the partition of the space made by $s$ is a binary tree where each cell in the partition is a rectangle of the form  $R_j=(l_1,u_1]\times\dots\times(l_p,u_p]\subset\mathbb{R}^p$ with $-\infty\leq l_m<u_m\leq\infty$ and $s(x)=j$ if  $x\in R_j$.

\section{Combined kernel and tree boosting}\label{ktboost}
Let $R^2(F_{m-1}+f)$ denote the functional, which is proportional to a second order Taylor approximation of the empirical risk in \eqref{emprisk} at the current estimate $F_{m-1}$:
\begin{equation}\label{apprRisk}
R^2(F_{m-1}+f)=\sum_{i=1}^n g_{m,i} f(x_i)+\frac{1}{2}h_{m,i} f(x_i)^2,
\end{equation}
where $g_{m,i}$ and $h_{m,i}$ are the functional gradient and Hessian of the empirical risk evaluated at the functions $F_{m-1}(x)$ and $I_{\{x=x_i\}}(x)$, where $I_{\{x=x_i\}}(x)=1$ if $x=x_i$ and $0$ otherwise:
\begin{equation}\label{gradHess}
\begin{split}
&g_{m,i}=\frac{\partial}{\partial F}L(y_i,F)\Big|_{F=F_{m-1}(x_i)},\\
&h_{m,i}=\frac{\partial^2}{\partial F^2}L(y_i,F)\Big|_{F=F_{m-1}(x_i)}.
\end{split}
\end{equation}

The KTBoost algorithm presented in Algorithm \ref{ktboost_algo} works as follows. In each boosting iteration, a candidate tree $f^T_m(x)$ and RKHS function $f^K_m(x)$ are found as minimizers of the second order Taylor approximation $R^2(F_{m-1}+f)$. This corresponds to applying one step of a functional version of Newton's method. It can be shown that candidate trees $f^T_m(x)$ can be found as weighted least squares minimizers; see, e.g., \citet{chen2016xgboost} or \citet{sigrist2018gradient}. Further, the candidate penalized RKHS regression functions $f^K_m(x)$ can be found as shown in Proposition \ref{kernelup} below. The KTBoost algorithm then selects either the tree or the RKHS function such that the addition of the base learner to the ensemble according to Equation \eqref{shrinkage} results in the lowest risk. Note that for the RKHS boosting part, the update equation $F_m(x)= F_{m-1}(x)+ \nu f_m(x)$ can be replaced by simply updating the coefficients $\alpha_m$.

\begin{algorithm}[ht!]
	\caption{KTBoost}\label{ktboost_algo}
	\begin{algorithmic}[1]
		\STATE Initialize $F_0(x)=\argmin_{c\in\mathbb{R}^d}R(c)$.
		\FOR{$m=1$ {\bfseries to} $M$}
		\STATE Compute the gradient $g_{m,i}$ and Hessian $h_{m,i}$ as defined in \eqref{gradHess}
		\STATE Find the candidate regression tree $f^T_m(x)$ and RKHS function $f^K_m(x)$
		\begin{itemize}
			\item[] $f^T_m(x)=\argmin_{f\in \mathcal{T}}R^2(F_{m-1}+f)$ 
			\item[] $f^K_m(x)=\argmin_{f\in \mathcal{H}}R^2(F_{m-1}+f)+\frac{1}{2} \lambda\| f\|^2_{\mathcal{H}}$
		\end{itemize}
		where the approximate risk $R^2(F_{m-1}+f)$ is defined in \eqref{apprRisk}
		\IF{$R\left(F_{m-1}+  \nu f^T_m(x)\right)\leq R\left(F_{m-1}+ \nu  f^K_m(x)\right)$}
		\STATE  $f_m(x)=f^T_m(x)$
		\ELSE  
		\STATE $f_m(x)=f^K_m(x)$
		\ENDIF
		\STATE Update $F_m(x)= F_{m-1}(x)+ \nu f_m(x)$
		\ENDFOR
	\end{algorithmic}
\end{algorithm}

If either the loss function is not twice differentiable in its second argument or the second derivative is zero or constant on a non-null set of the support of $X$, one can alternatively use gradient boosting. The gradient boosting version of KTBoost is obtained as a special case of the Algorithm \ref{ktboost_algo} by setting $h_{m,i}=1$. Gradient boosting has the advantage that it is computationally less expensive than Newton boosting since, in contrast to \eqref{wKernMat}, the kernel matrix does not depend on the iteration number $m$; see Section \ref{comcost} for more details.

\begin{proposition}\label{kernelup}
	The kernel ridge regression solution $f^K_m(x)$ in the regularized Newton boosting update step is given by $f^K_m(x)=k(x)^T\alpha_m,$ where $k(x)$ is defined in \eqref{kxeq} and
	\begin{equation}\label{wKernMat}
	\alpha_m=D_m\left(D_mKD_m+\lambda I_n\right)^{-1}D_my_m,
	\end{equation}
	where $D_m=\text{diag}\left(\sqrt{h_{m,i}}\right),$ $h_{m,i}>0$, $y_m=(-{g_{m,1}}/{h_{m,1}},\dots,-{g_{m,n}}/{h_{m,n}})^T,$ and $I_n$ is the identity matrix of dimension $n$.
\end{proposition}
\begin{proof}
	We have 
	\begin{equation*}
	\begin{split}
	\argmin_{f\in \mathcal{H}}&\sum_{i=1}^n g_{m,i} f(x_i)+\frac{1}{2}h_{m,i} f(x_i)^2+\frac{1}{2}\lambda \| f\|^2_{\mathcal{H}}\\
	&=\argmin_{f\in \mathcal{H}}\sum_{i=1}^n h_{m,i}\left(-\frac{g_{m,i}}{h_{m,i}}-f(x_i)\right)^2+\lambda \| f\|^2_{\mathcal{H}}\\
	&=\argmin_{\alpha}\|D_m{y}_m-D_m{K}\alpha\|^2+\lambda\alpha^TK\alpha.
	\end{split}
	\end{equation*}
If we take derivatives with respect to $\alpha$, equal them to zero, and solve for $\alpha$, we find that
	\begin{equation*}
	\begin{split}
	\alpha_m&=\left(KD_m^2K+\lambda K\right)^{-1}KD_m^2 y_m\\
	&=\left(D_m^2K+\lambda I_n\right)^{-1}D_m^2y_m\\
	&=D_m\left(D_mKD_m+\lambda I_n\right)^{-1}D_my_m.
	\end{split}
	\end{equation*}
\end{proof}


\subsection{Reducing computational costs for large data}\label{comcost}
Concerning the regression trees, finding the splits when growing the trees is the computationally demanding part. There are several approaches in the literature on how this can be done efficiently for large data; see, e.g., \citet{chen2016xgboost} or \citet{ke2017lightgbm}. The computationally expensive part for finding the kernel regression updates is the factorization of the kernel matrix which scales with $O(n^3)$ in time. There are several approaches that allow for computational efficiency in the large data case. Examples of this include low rank approximations based on, e.g., the Nystr{\"o}m method \citep{williams2001using} and extensions of it such as divide-and-conquer kernel ridge regression \citep{zhang2013divide,zhang2015divide}, early stopping of iterative optimization methods \citep{yao2007early,blanchard2010optimal, raskutti2014early,ma2017diving}, stochastic gradient descent \citep{cesa2004generalization, dai2014scalable}, random feature approximations \citep{rahimi2008random}, and compactly supported kernel functions \citep{gneiting2002compactly,bevilacqua2016estimation} which results in a sparse kernel matrix $K$ which can be efficiently factorized. 

Note that if gradient descent is used instead of Newton's method, the RKHS function $f^K_m(x)$ can be found efficiently by observing that, in contrast to \eqref{wKernMat}, the kernel matrix $K+\lambda I_n$ does not depend on the iteration number $m$, i.e., its inverse or a Cholesky factor of it needs to be calculated only once. Further, the two learners can be learned in parallel. 

In our empirical analysis, we use the Nystr{\"o}m method for dealing with large data sets. The Nystr{\"o}m method approximates the kernel $K(\cdot,\cdot)$ by first choosing a set of $l$ so-called Nystr{\"o}m samples $x^*_1,\dots, x^*_l$. Often these are obtained by sampling uniformly from the data. Denoting the kernel matrix that corresponds to these points as $K^*$, the Nystr{\"o}m method then approximates the kernel $K(\cdot,\cdot)$ as 
$$K(x,y)\approx k_l(x)^T{K^*_{l,l}}^{-1}k_l(y),$$
where $k_l(x)=(K(x,x^*_1),\dots,K(x,x^*_l))^T$ and $\left(K^*_{l,l}\right)_{j,k}=K(x^*_j,x^*_k)$, $1\leq j,k\leq l$. In particular, the reduced-rank Nystr{\"o}m approximation to the full kernel matrix $K$ is given by 
$$K\approx K^*_{n,l}{K^*_{l,l}}^{-1}{K^*_{n,l}}^T,$$ where $\left(K^*_{n,l}\right)_{j,k}=K(x_j,x^*_k)$, $1\leq j\leq n$, $1\leq k\leq l$.


\section{Experimental results}\label{exp}
\subsection{Simulation study}\label{simstudysec}
We first conduct a small simulation study to illustrate that the combination of discontinuous trees and continuous kernel machines can indeed better learn functions with both discontinuous and smooth parts. We consider random functions $F:[0,1]\rightarrow\mathbb{R}$ with five random jumps in $[0,0.5]$:
\begin{equation}
\begin{split}
F(x) =&\sum_{i=1}^{5}g_i\textbf{1}_{(t_{i},1]}(x)+\sin(8\pi x),\\& t_i \overset{\text{iid}}{\sim}  \text{Unif}(0,0.5), ~~ g_i\overset{\text{iid}}{\sim} \text{Unif}(0,5)
\end{split}
\end{equation}
and data according to $$y_i=F(x_i)+N(0,0.25^2), ~~x_i\overset{\text{iid}}{\sim} \text{Unif}(0,1),~~ i=1,\dots,1000.$$ In Figure \ref{simstudy}  on the left-hand side, an example of such a function and corresponding data is shown. We simulate 1000 times such random functions as well as training, validation, and test data of size $n=1000$. For each simulation run, learning is done on the training data. The number of boosting iterations is chosen on the validation data with the maximum number of boosting iterations being $M=1000$. We use a learning rate of $\nu=0.1$ as this is a reasonable default value \citep{buhlmann2007boosting} and trees of depth $=1$ as there are no interactions. Further, for the RKHS ridge regression, we use a Gaussian kernel 
\begin{equation}\label{rbf}
K(x_1,x_2)=\exp\left(-\|x_1-x_2\|^2/\rho^2\right),
\end{equation}
with $\rho=0.1$ and $\lambda=1$. In Figure \ref{simstudy} on the right-hand side, we show the pointwise test mean square error (MSE) for tree and kernel boosting as well as the combined KTBoost algorithm. We observe that tree boosting performs better than kernel boosting in the area where the discontinuities are located and, conversely, kernel boosting outperforms tree boosting on the smooth part. The figure also clearly shows that KTBoost outperforms both tree and kernel boosting as it achieves the MSE of tree boosting on the interval with jumps and the MSE of kernel boosting on the smooth part.

\begin{figure}[ht!]
	\centering
	\includegraphics[width=0.46\columnwidth]{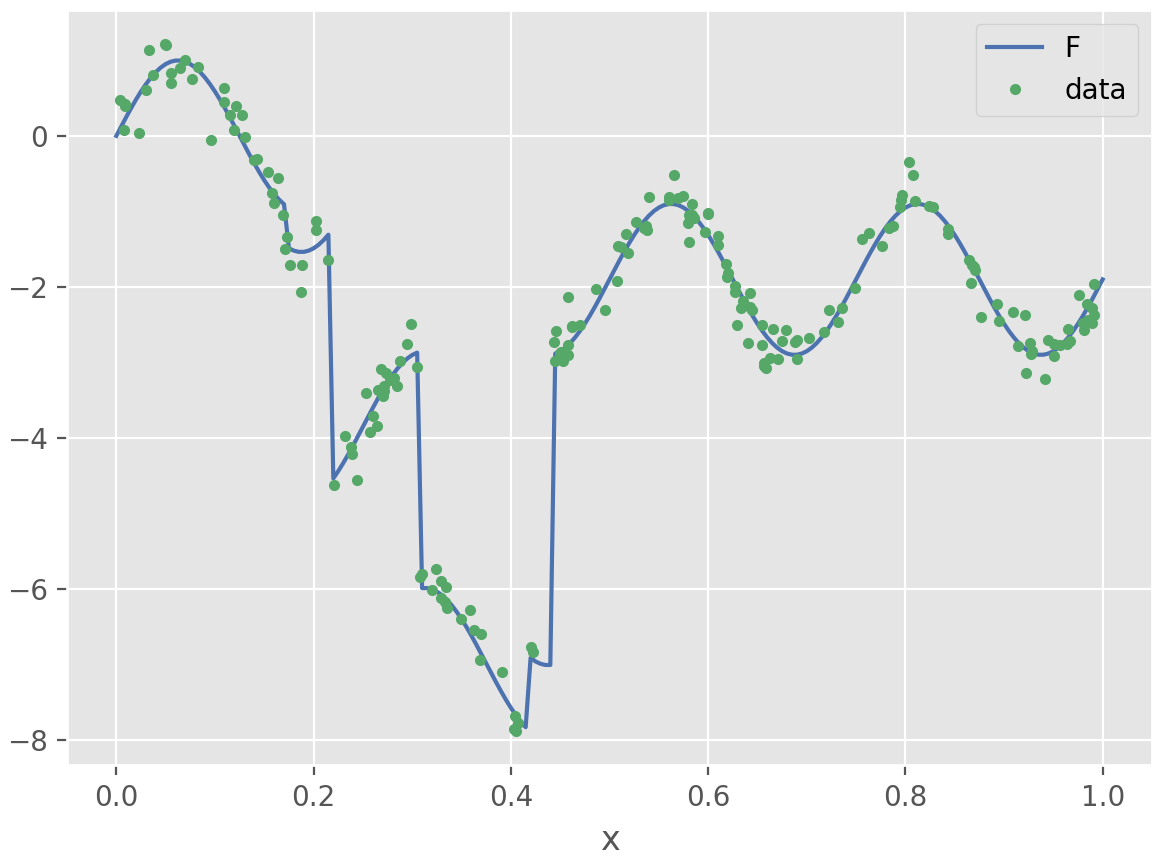}
	\includegraphics[width=0.49\columnwidth]{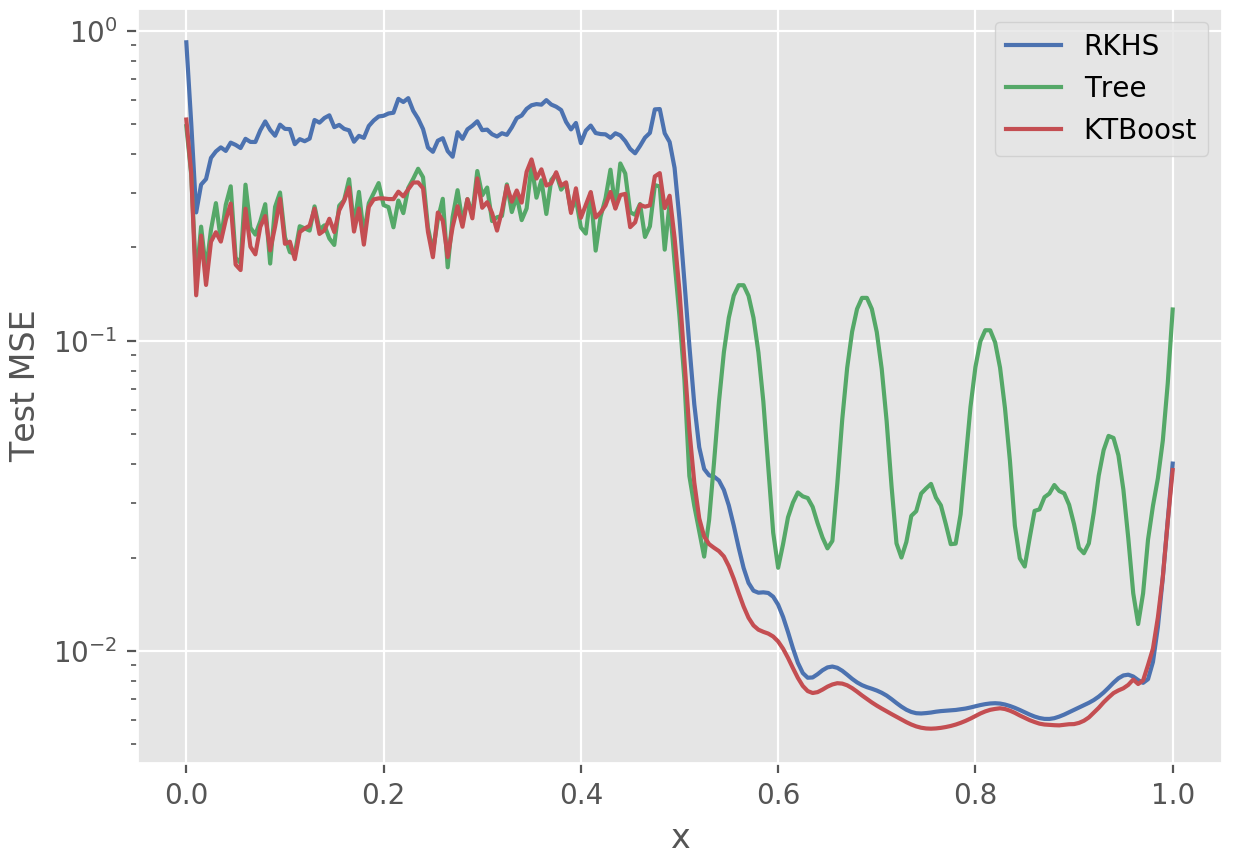}
	\caption{An example of a random function with five random jumps in $[0,0.5]$ and corresponding observed data (left plot) and pointwise mean square error (MSE) for tree and kernel boosting as well as the combined KTBoost algorithm (right plot).}
	\label{simstudy}
\end{figure}
For the purpose of illustration, we have considered a one-dimensional example. However, in practice discontinuities, or strong non-linearities, as well as smooth parts are likely to occur at the interaction level in higher dimensions of a feature space.

\subsection{Real-world data}\label{realworlddata}
In the following, we compare the KTBoost algorithm with tree and kernel boosting using the following Delve, Keel, Kaggle, and UCI data sets: abalone, ailerons, bank8FM, elevators, energy, housing, liberty, NavalT, parkinsons, puma32h, sarcos, wine, adult, cancer, ijcnn, ionosphere, sonar, car, epileptic, glass, and satimage. Detailed information on the number of samples and features can be found in Table \ref{data_sum}. We consider both regression as well as binary and multiclass classification data sets. Further, we include data sets of different sizes in order to investigate the performance on both smaller and larger data sets, as small- to moderately-sized data sets continue to be widely used in applied data science despite the recent focus on very large data sets in machine learning research. We use the squared loss for regression, the logistic regression loss for binary classification, and the cross-entropy loss with the softmax function for multiclass classification.
\begin{table}[ht!]
\centering
\caption{Summary of data sets.} 
\label{data_sum}
\begin{tabular}{llrr}
  \hline
\hline
Data & \# classes & Nb. samples & Nb. features \\ 
  \hline
abalone & regression & 4177 &  10 \\ 
  ailerons & regression & 13750 &  40 \\ 
  bank8FM & regression & 8192 &   8 \\ 
  elevators & regression & 16599 &  18 \\ 
  energy & regression & 768 &   8 \\ 
  housing & regression & 506 &  13 \\ 
  liberty & regression & 50999 & 117 \\ 
  NavalT & regression & 11934 &  16 \\ 
  parkinsons & regression & 5875 &  16 \\ 
  puma32h & regression & 8192 &  32 \\ 
  sarcos & regression & 48933 &  21 \\ 
  wine & regression & 4898 &  11 \\ 
   \hline
adult & 2 & 48842 & 108 \\ 
  cancer & 2 & 699 &   9 \\ 
  ijcnn & 2 & 141691 &  22 \\ 
  ionosphere & 2 & 351 &  34 \\ 
  sonar & 2 & 208 &  60 \\ 
  car & 4 & 1728 &  21 \\ 
  epileptic & 5 & 11500 & 178 \\ 
  glass & 7 & 214 &   9 \\ 
  satimage & 6 & 6438 &  36 \\ 
   \hline
\hline
\end{tabular}
\end{table}

For the regression data sets, we use gradient boosting, and for the classification data sets, we use boosting with Newton updates since this can result in more accurate predictions \citep{sigrist2018gradient}. For some classification data sets (adult, ijcnn, epileptic, and satimage), Newton boosting is computationally infeasible on a standard single CPU computer with the current implementation of KTBoost, despite the use of the Nystr{\"o}m method with a reasonable number of Nystr\"om samples, say 1000, since the weighted kernel matrix in Equation \eqref{wKernMat} needs to be factorized in every iteration. We thus also use gradient boosting for these data sets. Technically, it would be possible for these cases to learn the trees using Newton's method, or using the hybrid gradient-Newton boosting version of \citet{friedman2001greedy}, but this would result in an unfair comparison that is biased in favor of the base learner which is learned with the better optimization method. For the larger data sets (liberty, sarcos, adult, ijcnn), we use the Nystr{\"o}m method described in Section \ref{comcost}. Specifically, we use $l=1000$ Nystr{\"o}m samples, which are uniformly sampled from the training data. In general, the larger the number of Nystr{\"o}m samples, the lower the approximation error but the higher the computational costs. \citet{williams2001using} reports good results with $l\approx1000$ for several data sets. All calculations are done with the Python package \texttt{KTBoost} on a standard laptop with a 2.9 GHz quad-core processor and 16GB of RAM.


All data sets are randomly split into three non-overlapping parts of equal size to obtain training, validation and test sets. Learning is done on the training data, tuning parameters are chosen on the validation data, and model comparison is done on the holdout test data. All input features are standardized using the training data to have approximately mean zero and variance one. In order to measure the generalization error and approximately quantify variability in it, we use ten different random splits of the data into training, validation and test sets. We note that when using a resampling approach, standard statistical tests, such as a paired t-test, cannot be used to do a pairwise comparison of the different algorithms on a dataset basis since training and test datasets in different splits are dependent due to overlap \citep{dietterich1998approximate, bengio2004no, demvsar2006statistical}. In particular, this can result in biased standard error estimates for the generalization error.

For the RKHS ridge regression, we use again a Gaussian kernel; see Equation \eqref{rbf}. Concerning tuning parameters,  we select the number of boosting iterations $M$ from $\{1,2,\dots, 1000\}$, the learning rate $\nu$ from $\{1,10^{-1},10^{-2},10^{-3}\}$, the maximal depth of the trees from $\{1,5,10\}$, and the kernel ridge regularization parameter $\lambda$ from $\{1,10\}$. Further, the kernel range parameter $\rho$ is chosen using $k$-nearest neighbors distances as described in the following. We first calculate the average distance of all $k$-nearest neighbors in the training data, where $k$ is a tuning parameter selected from $\{5,50,500,5000,n-1\}$ and $n$ is the size of the training data. We then choose $\rho$ such that the kernel function has decayed to a value of $0.01$ at this average $k$-nearest neighbors distance. This is motivated by the fact that for a corresponding Gaussian process with such a covariance function, the correlation has decayed to a level of 1\% at this $k$-nearest neighbor distance. If the training data contains less than $5000$ (or $500$) samples, we use $n-1$ as the maximal number for the $k$-nearest neighbors. In addition, we include $\rho$ which equals the average $(n-1)$-nearest neighbor distance. The latter choice is done in order to also include a range which results in a kernel that decays slowly over the entire space. For the large data sets where the Nystr{\"o}m method is used, we calculate the average $k$-nearest neighbors distance based on the Nystr{\"o}m samples. I.e., in this case, the maximal $k$ equals $l-1$.

\begin{table}[ht!]
\centering
\caption{Comparison of KTBoost with tree and kernel boosting using test mean square error 
          (regression) and test error rate (classification). 
          In parentheses are approximate standard deviations.
          Below are average ranks of the methods over the different datasets. 
          A p-value of a Friedman test with an Iman and Davenport correction for comparing the different algorithms is also reported.
          The last row shows Holm-Bonferroni corrected p-values of sign tests for pairwise 
          comparison of the KTBoost algorithm with tree and kernel boosting.} 
\label{res_save}
\begingroup\footnotesize
\begin{tabular}{ll|l|l}
  \hline
\hline
Data & KTBoost & Tree & Kernel \\ 
  \hline
abalone & 4.65~~(0.248) & 5.07~~(0.261) & \textbf{4.64}~~(0.255) \\ 
  ailerons & \textbf{2.64e-08}~~(6.19e-10) & 8.11e-08~~(2.39e-09) & 2.64e-08~~(6.19e-10) \\ 
  bank8FM & \textbf{0.000915}~~(4.02e-05) & 0.000945~~(2.47e-05) & 0.000945~~(5.83e-05) \\ 
  elevators & \textbf{4.83e-06}~~(2.9e-07) & 5.66e-06~~(1.44e-07) & 5.18e-06~~(3.89e-07) \\ 
  energy & \textbf{0.282}~~(0.0372) & 0.335~~(0.093) & 1.3~~(0.377) \\ 
  housing & \textbf{12.7}~~(3.19) & 15.1~~(3.23) & 13.6~~(2.51) \\ 
  liberty & \textbf{14.5}~~(0.323) & 14.5~~(0.314) & 15.2~~(0.345) \\ 
  NavalT & \textbf{6.51e-09}~~(1.15e-09) & 1.15e-06~~(1.58e-07) & 6.51e-09~~(1.15e-09) \\ 
  parkinsons & \textbf{73.3}~~(1.98) & 81.1~~(2.44) & 73.3~~(1.91) \\ 
  puma32h & \textbf{6.5e-05}~~(2.27e-06) & 6.51e-05~~(2.13e-06) & 0.000695~~(2.2e-05) \\ 
  sarcos & \textbf{7.99}~~(0.206) & 9.6~~(0.207) & 17.8~~(0.586) \\ 
  wine & \textbf{0.444}~~(0.012) & 0.471~~(0.0169) & 0.506~~(0.0106) \\ 
   \hline
adult & 0.128~~(0.00295) & \textbf{0.128}~~(0.00313) & 0.163~~(0.00512) \\ 
  cancer & 0.0362~~(0.00744) & 0.0415~~(0.0153) & \textbf{0.0358}~~(0.0107) \\ 
  ijcnn & \textbf{0.0122}~~(0.000685) & 0.0123~~(0.000702) & 0.0387~~(0.00516) \\ 
  ionosphere & \textbf{0.0872}~~(0.017) & 0.103~~(0.0226) & 0.107~~(0.0239) \\ 
  sonar & 0.194~~(0.0394) & 0.223~~(0.05) & \textbf{0.193}~~(0.0491) \\ 
  car & \textbf{0.0399}~~(0.00505) & 0.0411~~(0.00685) & 0.041~~(0.00624) \\ 
  epileptic & \textbf{0.354}~~(0.00612) & 0.373~~(0.00614) & 0.442~~(0.0265) \\ 
  glass & \textbf{0.308}~~(0.0711) & 0.315~~(0.0589) & 0.344~~(0.0581) \\ 
  satimage & \textbf{0.089}~~(0.00452) & 0.112~~(0.00504) & 0.0903~~(0.00417) \\ 
   \hline
Average rank & 1.24 & 2.48 & 2.29 \\ 
   \hline
p-val Friedman test & 7.84e-06 &  &  \\ 
  Adj. p-val sign test &  & 6.29e-05 & 0.00885 \\ 
   \hline
\hline
\end{tabular}
\endgroup
\end{table}

The results are shown in Table \ref{res_save}. For the regression data sets, we show the average test mean square error (MSE) over the different sample splits, and for the classification data sets, we calculate the average test error rate (=misclassification rate). The numbers in parentheses are approximate standard deviations over the different sample splits. In the last row, we report the average rank of every method over the different data sets. We find that KTBoost achieves higher predictive accuracy than both tree and kernel boosting for the large majority of data sets. Specifically, KTBoost has an average rank of $1.24$ and achieves higher predictive accuracy than both tree and kernel boosting for seventeen out of twenty-one data sets. A Friedman test with an Iman and Davenport correction \citep{iman1980approximations} gives a p-value of $7.84\cdot10^{-6}$ which shows that the differences in the three methods are highly significant. We next assess whether the pairwise differences in accuracy between the different methods are statistically significant using a sign test. Further, we apply a Holm-Bonferroni correction \citep{holm1979simple} to account for the fact that we do multiple tests. Despite the sign test having low power and the application of the conservative Holm-Bonferroni correction, KTBoost is highly significantly better than both tree and kernel boosting with adjusted p-values below $0.01$. The difference between kernel and tree boosting is not significant with both the adjusted and non-adjusted p-values being above $0.1$ (result not tabulated).



Note that we do not report the optimal tuning parameters since this is infeasible for all combinations of data sets and sample splits, and aggregate values are not meaningful since different tuning parameters often compensate each other in a non-linear way (e.g., number of iterations, learning rate, and tree depth or kernel regularization $\lambda$). Further, it is also difficult to concisely summarize the composition of the ensembles in terms of different base learners as a base learner that is added in an earlier boosting stage is more important than one that is added in a later stage \citep{buhlmann2003boosting}, and the properties of the base learners also depend on the chosen tuning parameters. We also note that one can also consider additional tuning parameters. For trees, this includes the minimal number of samples per leaf, row and column sub-sampling, and penalization of leave values, and for the kernel regression, this includes the smoothness of the kernel function, or, in general, the class of kernel functions. One could also use different learning rates for the two types of base learners. Due to limits on computational costs, we have not considered all possible choices and combinations of tuning parameters. However, it is likely that a potential increase in predictive performance in either tree or kernel boosting will also result in an increase in accuracy of the combined KTBoost algorithm. We also note that in our experimental setup, the tuning parameter grid for the KTBoost algorithm is larger compared to the tree and kernel boosting cases. This seems inevitable in order to allow for the fairest possible comparison, though. Restricting one type of tuning parameters for the combined version but not for the single base learner case seems to be no alternative. Somewhat alleviating this concern is the fact that, in the above simulation study, we also find outperformance when not choosing tuning parameters using cross-validation, and on the downside, a larger tuning parameter grid might potentially also lead to overfitting. Finally, we remark that we have also considered to compare the risk of the un-damped base learners 
$$R\left(F_{m-1}+f^T_m(x)\right)\leq R\left(F_{m-1}+ f^K_m(x)\right)$$ 
in line $5$ of Algorithm \ref{ktboost_algo} when selecting the base learners that is added to the ensemble, and we obtain very similar results (see supplementary material).

\section{Conclusions}\label{conclusion}
We have introduced a novel boosting algorithm, which combines trees and RKHS functions as base learners. Intuitively, the idea is that discontinuous trees and continuous RKHS functions complement each other since trees are better suited for learning rougher parts of functions and RKHS regression functions can better learn smoother parts of functions. We have compared the predictive accuracy of the KTBoost algorithm with tree and kernel boosting and have found that KTBoost achieves significantly higher predictive accuracy compared to tree and kernel boosting.

Future research can be done in several directions. First, it would be interesting to investigate to which extent other base learners such as neural networks \citep{huang2018,nitanda2018} are useful in addition to trees and kernel regression functions. Generalizing the KTBoost algorithm using reproducing kernel Kre\u{\i}n space (RKKS) learners \citep{ong2004learning,oglic2018learning} instead of RKHS learners can also be investigated. Further, theoretical results such as learning rates or bounds on the risk could help to shed further insights on why the combination of trees and kernel machines leads to increased predictive accuracy. Finally, it would be interesting to compare the KTBoost algorithm on very large data sets using different strategies for reducing the computational complexity of the RKHS part. Several potential strategies on how this can be done are briefly outlined in Section \ref{comcost}.

\section*{Acknowledgements}
We would like to thank Christoph Hirnschall for helpful feedback on the article. This research was partially supported by the Swiss Innovation Agency - Innosuisse (25746.1 PFES-ES).


\bibliography{bibfile_KTBoost}

\clearpage

\section*{Supplementary material}

\subsection*{Results when the shrinkage parameter is not used for choosing the base learner}
\begin{table}[ht!]
\centering
\caption{Comparison of KTBoost with tree and kernel boosting using test mean square error 
          (regression) and test error rate (classification). 
          In parentheses are approximate standard deviations.
          Below are average ranks of the methods over the different datasets. 
          A p-value of a Friedman test with an Iman and Davenport correction for comparing the different algorithms is also reported.
          The last row shows Holm-Bonferroni corrected p-values of sign tests for pairwise 
          comparison of the KTBoost algorithm with tree and kernel boosting.} 
\label{res_saveNoShrinkage}
\begingroup\footnotesize
\begin{tabular}{ll|l|l}
  \hline
\hline
Data & KTBoost & Tree & Kernel \\ 
  \hline
abalone & \textbf{4.63}~~(0.246) & 5.07~~(0.261) & 4.64~~(0.255) \\ 
  ailerons & 2.64e-08~~(6.19e-10) & 8.11e-08~~(2.39e-09) & \textbf{2.64e-08}~~(6.19e-10) \\ 
  bank8FM & \textbf{0.000908}~~(3.89e-05) & 0.000945~~(2.47e-05) & 0.000945~~(5.83e-05) \\ 
  elevators & \textbf{4.71e-06}~~(2.16e-07) & 5.66e-06~~(1.44e-07) & 5.18e-06~~(3.89e-07) \\ 
  energy & \textbf{0.28}~~(0.039) & 0.335~~(0.093) & 1.3~~(0.377) \\ 
  housing & \textbf{12.8}~~(3.42) & 15.1~~(3.23) & 13.6~~(2.51) \\ 
  liberty & \textbf{14.5}~~(0.299) & 14.5~~(0.314) & 15.2~~(0.345) \\ 
  NavalT & 6.51e-09~~(1.15e-09) & 1.15e-06~~(1.58e-07) & \textbf{6.51e-09}~~(1.15e-09) \\ 
  parkinsons & 73.4~~(2.02) & 81.1~~(2.44) & \textbf{73.3}~~(1.91) \\ 
  puma32h & 6.51e-05~~(2.27e-06) & \textbf{6.51e-05}~~(2.13e-06) & 0.000695~~(2.2e-05) \\ 
  sarcos & \textbf{7.71}~~(0.224) & 9.6~~(0.207) & 17.8~~(0.586) \\ 
  wine & \textbf{0.463}~~(0.0124) & 0.471~~(0.0169) & 0.506~~(0.0106) \\ 
   \hline
adult & \textbf{0.128}~~(0.00304) & 0.128~~(0.00313) & 0.163~~(0.00512) \\ 
  cancer & \textbf{0.0349}~~(0.00873) & 0.0415~~(0.0153) & 0.0358~~(0.0107) \\ 
  ijcnn & \textbf{0.0122}~~(0.000567) & 0.0123~~(0.000702) & 0.0387~~(0.00516) \\ 
  ionosphere & \textbf{0.0855}~~(0.018) & 0.103~~(0.0226) & 0.107~~(0.0239) \\ 
  sonar & 0.194~~(0.0394) & 0.223~~(0.05) & \textbf{0.193}~~(0.0491) \\ 
  car & \textbf{0.0405}~~(0.0059) & 0.0411~~(0.00685) & 0.041~~(0.00624) \\ 
  epileptic & \textbf{0.351}~~(0.00543) & 0.373~~(0.00614) & 0.442~~(0.0265) \\ 
  glass & 0.322~~(0.0702) & \textbf{0.315}~~(0.0589) & 0.344~~(0.0581) \\ 
  satimage & \textbf{0.0892}~~(0.00399) & 0.112~~(0.00504) & 0.0903~~(0.00417) \\ 
   \hline
Average rank & 1.29 & 2.43 & 2.29 \\ 
   \hline
p-val Friedman test & 5.48e-05 &  &  \\ 
  Adj. p-val sign test &  & 0.000664 & 0.0144 \\ 
   \hline
\hline
\end{tabular}
\endgroup
\end{table}

\clearpage


\end{document}